\documentclass[12pt]{article}
\usepackage{graphicx}%
\usepackage{multirow}%
\usepackage{amsmath,amssymb,amsfonts,mathtools,amsthm}%
\usepackage{booktabs}
\usepackage{algorithm}%
\usepackage{algpseudocode}%
\usepackage{natbib}
\usepackage{varwidth}
\usepackage{orcidlink,authblk}

\newcommand{\vecx}{\mathbf{x}}
\newcommand{\vecy}{\mathbf{y}}

\newcommand{\sampcov}{\mathbf{S}}
\newcommand{\vectheta}{\boldsymbol\theta}
\newcommand{\vecvartheta}{\boldsymbol\vartheta}
\newcommand{\vecSigma}{\boldsymbol\Sigma}
\newcommand{\vecmu}{\boldsymbol\mu}

\DeclarePairedDelimiter\abss{\lvert}{\rvert}

\newcommand{\argmax}{\arg\max}

\theoremstyle{thmstyleone}%
\newtheorem{theorem}{Theorem}
\newtheorem{lemma}[theorem]{Lemma}%

\theoremstyle{thmstyletwo}%
\newtheorem{remark}{Remark}%

\theoremstyle{thmstylethree}%
\newtheorem{definition}{Definition}%

\title{funOCLUST: Clustering Functional Data with Outliers}
\author[1]{Katharine M. Clark\orcidlink{0000-0002-6162-2300}}
\author[2]{Paul D. McNicholas\orcidlink{0000-0002-2482-523X}}

\affil[1]{\small Department of Mathematics \& Statistics, Trent University, Ontario, Canada.}
\affil[2]{\small Department of Mathematics \& Statistics, McMaster University, Ontario, Canada.}

\date{} 

\begin{document}

\maketitle

\begin{abstract}
Functional data present unique challenges for clustering due to their infinite-dimensional nature and potential sensitivity to outliers. An extension of the OCLUST algorithm to the functional setting is proposed to address these issues. The approach leverages the OCLUST framework, creating a robust method to cluster curves and trim outliers.  The methodology is evaluated on both simulated and real-world functional datasets, demonstrating strong performance in clustering and outlier identification. \\[-10pt]
	
\end{abstract}
%
%
\section{Introduction}

Functional data analysis presents a unique challenge due to the nature of the data. While functions are observed at discrete time points, they exist in an infinite-dimensional function space, which requires methods that respect their smoothness, continuity, and underlying structure rather than treating them as simple multivariate observations. A natural secondary aim of functional data analysis is to group similar curves or trajectories using cluster analysis, e.g. clustering girls' versus boys' growth curves  and weather stations based on recorded temperatures \citep{ramsay05}. 

Clustering is a subtype of classification, i.e., unsupervised classification, which aims to group similar sets of observations together without any prior knowledge of the cluster memberships. While non-parametric methods exist, e.g., hierarchical and $k$-means clustering, model-based clustering employs a parametric approach. In general, the density of a finite mixture model is
\begin{equation}
	f(\vecx\mid\vecvartheta)=\sum_{g=1}^{G}\pi_g f_g(\vecx\mid\vectheta_g),
	\label{eq:generalmixture2}
\end{equation}
where $\vecvartheta=\{\pi_1, \dots, \pi_G, \vectheta_1, \dots \vectheta_G\}$, $\pi_g>0$ is the $g$th mixing proportion with $\sum_{g=1}^G \pi_g =1$, and $f_g(\vecx\mid\vectheta_g)$ is the $g$th component density with parameters $\vectheta_g$. Typically, each component corresponds to a cluster \citep[see][for a discussion]{mcnicholas16a}. While each component can be modelled with various distributions, the Gaussian distribution remains popular due to its simplicity. The parameters and component memberships are estimated by maximizing the log-likelihood, often with the expectation-maximization (EM) algorithm \citep{dempster77}.

While model-based clustering methods for functional data have been developed \citep[e.g.,][]{james03, bouveyron11b}, they remain sensitive to noisy observations. Alternatives with more robust distributions \citep{anton23, amovin22} and trimming approaches \citep{rivera19} have been proposed, but these methods remain focused on subspace-specific clustering. While subspace methods can be highly effective, there are settings in which the clustering structure requires the entire functional domain, making dimension reduction less desirable.

Non-parametric trimming methods do exist \citep[e.g.,][]{garcia05} but there is a paucity of trimming algorithms within the model-based functional data clustering domain. The proposed algorithm, funOCLUST, combines trimming with model-based clustering to detect outliers and cluster similar functions simultaneously while retaining information from the full functional representation.

\section{Preliminaries}
\subsection{OCLUST Algorithm}

In the multivariate normal data domain, \cite{clark24} develop the OCLUST algorithm, which simultaneously clusters data and trims outliers. Data are modelled with mixtures of Gaussian distributions, and outliers are trimmed iteratively one-by-one. A subset log-likelihood is defined to be the log-likelihood of the data with a single data point removed, and there are $n$ such subsets. After each outlier removal, the distribution of the subset log-likelihoods is compared to the mixture of shifted and scaled beta distributions derived in the paper. Specifically, let $l_{\mathcal{X}}$ be the complete-data log-likelihood of the entire dataset, let $l_{\mathcal{X} \setminus \vecx_j}$ be the complete-data log-likelihood of the $j$th subset, i.e., with observation $\vecx_j$ removed, and let $D_j=l_{\mathcal{X} \setminus \vecx_j}-l_{\mathcal{X}}$. Then, if $\vecx_j$ belongs to the $h$th component, 
\begin{equation}
	\label{eq:L}
	D_j \mid (z_{jh}=1) \sim f_{\text{beta}}\left(\frac{2n_h}{(n_h-1)^2} (d_j-c)~\bigg|~\frac{p}{2},\frac{n_h-p-1}{2}\right)
\end{equation}
for $c<d_j<\frac{(n_h-1)^2}{2n_h}+c, n_h>p+1$, where
$$z_{jh}=\begin{cases}1 & \text{if }\vecx_j \text{belongs to component}~h,\\
	0&\text{otherwise,}\end{cases}$$ 
$n_h$ is the number of points in component $h$,
$c=-\log\hat{\pi}_h+\frac{p}{2}\log(2\pi)+\frac{1}{2}\log\abss{\sampcov_h}$, $\hat{\pi}_h=n_h/n$, $$\sampcov_h=\frac{1}{n_h-1}\sum_{i=1}^n z_{ih}(\vecx_i-\bar{\vecx}_h)(\vecx_i-\bar{\vecx}_h)^\top$$ is the sample covariance matrix of component $h$, and $\bar{\vecx}_h=\frac{1}{n_h}\sum_{i=1}^n z_{ih}\vecx_i.$ The trimmed model chosen minimizes the Kullback-Leibler (KL) divergence between the observed differences and the derived distribution. 

\subsection{Functional Decomposition}
In multivariate data analysis, a random sample $(\vecx_1, \ldots \vecx_n)$ is typically a collection of vectors in $\mathbb{R}^p$. In functional data, a random sample of functions ($Y_1(t), \ldots Y_n(t)), t\in \mathcal{T} \subset \mathbb{R}$ exists in infinite-dimensional space. In practice, however, these functions are evaluated at a series of discrete values $\{t_{i1},\ldots, t_{ij_{i}}\}$, resulting in observations of the form $(\mathbf{t}_i, \vecy_i)$, where $\mathbf{t}_i=(t_{i1},\ldots, t_{ij_{i}}) $ is a vector of points in the domain and $\vecy_i=(y_{i1},\ldots, y_{ij_{i}})$ is the vector of observed values at the points $\mathbf{t}_i$. To return to infinite-dimensional space, we can reconstruct the functional form. A popular choice is a basis expansion, which is well explained by \cite{ramsay05}.

In a basis expansion, the function $Y(t)$ is approximated by a linear combination of basis functions so that $\hat Y(t)=\sum_k \beta_k B_k(t)$,
where $B_k(t)$ is the value at $t$ for the $k$th basis function. Basis functions must be independent from each other and must be able to approximate any function with reasonable accuracy. Some choices for basis functions are Fourier series, wavelet functions, power series, exponential, polynomial and power bases, as well as B-spline basis functions. 

Functional principal components analysis (FPCA) is a special type of basis expansion, where $Y(t)$ is approximated by a linear combination of mutually orthogonal and normalized weight functions that explain decreasing amounts of variation in the sample ($Y_1(t), \ldots Y_n(t))$. 

\subsection{Clustering Functional Data}\label{sec:funclust}
\cite{jacques14} organize functional data clustering into three paradigms: two-stage methods, nonparametric methods, and model-based methods. In model-based approaches, the filtering and clustering steps are completed simultaneously. \cite{james03} introduce a Gaussian mixture model for functional clustering by assuming that the coefficients in the basis expansion are normally-distributed. These coefficients are not fixed; rather, they are estimated within the EM algorithm. Similarly, \cite{jacques13} generate a Gaussian mixture model with FPCA scores. 

\cite{abraham03} propose a two-stage approach, first filtering the data with a B-spline basis and then clustering the resulting coefficients with the $k$-means algorithm. In this case, because the filtering is completed first, the coefficients are fixed. Similarly, \cite{peng08} first use FPCA to decompose the functions. The FPCA scores are then scaled and clustered using $k$-means clustering.  \cite{nguyen18} propose a two-part algorithm, first fitting the B-spline basis coefficients to the functions, and then applying a Gaussian mixture model to those coefficients.  \cite{bouveyron11b} also use a basis expansion but clustering is performed through a model in which component-specific subspaces are estimated along with the model parameters, thus reducing the dimension when there are large numbers of basis functions. The resultant algorithm is called funHDDC. 

More recently, \cite{pulido23} propose a methodology that transforms univariate functional data into a six-dimensional vector representation. They first fit a cubic spline basis, then calculate epigraph and hypograph indices for each curve and its first and second derivatives, and then cluster these indices. \cite{pulido25} extend the methodology to multivariate curves by defining these indices for multivariate functional data, whereby multiple variables are observed at each value in $\mathbf{t}_i$.

\subsection{Outliers in Functional Data Clustering}
\cite{hubert15} separate functional outliers into two types: isolated outliers and persistent outliers. While isolated outliers are atypical in a narrow region, e.g., a spike, persistent outliers affect a larger portion of the domain, e.g., shift, amplitude, or shape. In functional data clustering, algorithms that handle outliers also belong to the three paradigms in Section~\ref{sec:funclust}. Using a two-stage approach, \cite{garcia05} propose a method to cluster functional data with outliers by first decomposing the functions using a B-spline basis, and then clustering the resulting coefficients with trimmed $k$-means clustering. 

In the model-based setting, \cite{rivera19} use a funHDDC-type model to integrate the component-specific subspace concept with trimming by maximizing a trimmed pseudo log-likelihood. The funHDDC approach has also been extended to accommodate more flexible distributions, improving its robustness to outliers. The C-funHDDC method combines a contaminated mixture model with funHDDC by assuming that the functions reside in a lower-dimensional component-specific subspace \citep{amovin22, smith22}. Similarly, \cite{anton23} use a t-distribution with the funHDDC foundation to generate a more robust model in the presence of outliers, called T-funHDDC.

In a non-parametric approach, a functional isolation forest \citep[FIF;][]{staerman19} is a collection of isolation trees which randomly split the functional space into partitions, iteratively, until each function is isolated. Functions isolated in fewer iterations tend to be more outlying. An outlier score is calculated using a forest of these isolation trees. Finally, \cite{hubert15} determine functional outlyingness by first calculating outlyingness at each time point and then taking the weighted average to create an outlier score for each curve. 

\section{Methodology}\label{chap:fun}
\subsection{Overview}
The OCLUST algorithm hinges on the fact that log-likelihoods of the data subsets have an approximate shifted and scaled beta distribution. Extending the OCLUST algorithm to functional data will require something similar: a model with an associated log-likelihood function that can be evaluated on data subsets, and for which the distribution of these subset log-likelihoods can be derived.

While functions exist in infinite-dimensional space, they are observed at a series of discrete points $(t_{1}, ..., t_j)$. If each curve is sampled at the same points, then the data for each curve can be considered as a vector. If these vectors arise from a finite Gaussian mixture, then there will be a log-likelihood for the data and associated component parameters $\pi_g, \vecmu_g, \vecSigma_g$, for $g=1,\ldots,G$. An obvious choice would be to apply the OCLUST algorithm to the discretized values when modelling a finite Gaussian mixture model of functional data with outliers. However, these vectors will be the same dimension as the number of points at which each function is sampled. To retain enough information about the curve, the vectors of observations are likely to be very high-dimensional.

Instead, consider applying a filtering approach to the raw data, specifically their decomposition with cubic B-splines,
\begin{equation}
	Y_i(t)=\sum_{k=1}^{K+4} \beta_{i,k} B(t)+\varepsilon_i= \mathbf{B}(t) \boldsymbol{\beta}_i+\varepsilon_i,
\end{equation} 
where $\mathbf{B}(t)$ is the cubic B-spline basis with $K$ interior knots evaluated at point $t$, $\boldsymbol{\beta}_i$ is the vector of coefficients for the basis expansion of function $Y_i(t)$, and $\varepsilon_i\sim N(0,\sigma^2)$ is the error term. This retains the infinite-dimensional properties of the functions, and each function has a unique $(K+4)$-dimensional representation given by the vector $\boldsymbol{\beta}_i$. In effect, each $Y_i(t)$ is transformed to a vector in $\mathbb{R}^{K+4}$, the dimension of which is governed only by the choice in number of interior knots. Transforming $Y_i(t)$ maintains the principle of working with functions while creating a vector representation in a much more manageable dimension. If we assume that the vectors $\mathbf{\beta}_i$ come from a multivariate normal distribution, then the extension of the OCLUST algorithm to functional data becomes straightforward. 

\subsection{Subsets and the Distribution of Log-Likelihoods}\label{sec:fundist}
Consider a sample of i.i.d.\ functions ($Y_1(t), \ldots, Y_n(t)$). Suppose that each $Y_i(t)$ can be decomposed into a linear combination of cubic B-splines, i.e., $Y_i(t)=\mathbf{B}(t) \mathbf{\beta}_i+\varepsilon_i$, where $\varepsilon_i\sim N(0,\sigma^2)$. To capture heterogeneity across functional observations and enable clustering of similar functional patterns, \cite{nguyen18} propose that each $\mathbf{\beta}_i$ arises from a finite Gaussian mixture model with density
\begin{equation}
	f_{\mathbf{\beta}}(\mathbf{x} \mid \vecvartheta)=\sum_{g=1}^G \pi_g\phi(\mathbf{x}\mid \vecmu_g, \vecSigma_g).
\end{equation}
In addition, they show that the fitted coefficients ($\mathbf{b}_1,\ldots,\mathbf{b}_n$), corresponding to functions ($Y_1(t), \ldots, Y_n(t)$), follow their own Gaussian mixture model with density

\begin{equation}\label{eq:fitbmix2}
	f_{\hat{\mathbf{\beta}}}(\mathbf{b} \mid \vecvartheta)=\sum_{g=1}^G \pi_g\phi(\mathbf{b}\mid \vecmu_g,\vecSigma_g+\sigma^2(\mathbf{B}^\top\mathbf{B})^{-1}),
\end{equation}
where $\mathbf{B}=\mathbf{B}(\mathbf{t})$ is the matrix of basis functions evaluated at points $\mathbf{t}=(t_1,\ldots,t_j)$, the fitted functions are $\hat{Y}_i(t)=\mathbf{B}\mathbf{b}_i$, and the OLS-estimated coefficients are given by $\mathbf{b}_i=(\mathbf{B}^\top\mathbf{B})^{-1}\mathbf{B}^\top\mathbf{Y}_i$, with $\mathbf{Y}_i=(Y_i(t_1),\ldots, Y_i(t_j))$. Note that all functions $Y_i(t)$ are observed at the same time points $\mathbf{t}=(t_1,\ldots,t_j)$.

These principles allow us to formulate Lemma~\ref{theo:fun}. Consider the complete set of fitted coefficients $\mathcal{X}=\{\mathbf{b}_1,\ldots,\mathbf{b}_n\}$ generated by decomposing the functions $(Y_1(t), \ldots, Y_n(t))$ with a cubic B-spline basis with $K$ interior knots. Define the $j$th subset $$\mathcal{X}\setminus \mathbf{b}_j=\{\mathbf{b}_1, \ldots, \mathbf{b}_{j-1},\mathbf{b}_{j+1}, \ldots, \mathbf{b}_{n}\}$$ as the set of fitted coefficients for the functions $(Y_1(t), \ldots, Y_{j-1}(t),Y_{j+1}(t), \ldots, Y_{n}(t))$, i.e., the entire sample with the $j$th function removed. Let $l_{\mathcal{X}}$ denote the complete-data log-likelihood, i.e.,
\begin{equation}
	l_{\mathcal{X}}=\sum_{i=1}^n \sum_{g=1}^G z_{ig}\left[\log\pi_g + \log\phi(\mathbf{b}_i\mid \vecmu_g, \vecSigma_g)\right],
	\label{eq:cdlogb}
\end{equation}
where the complete-data comprise the coefficients $\{\mathbf{b}_1,\ldots,\mathbf{b}_n\}$ and the component memberships $\{\mathbf{z}_1,\ldots,\mathbf{z}_n\}$. Note that $z_{ig}=1$ if coefficient $\mathbf{b}_i$ belongs to component $g$, and $z_{ig}=0$ otherwise. 
\begin{lemma}\label{theo:fun}
	Consider a vector of fitted coefficients $\mathbf{b}_j$ generated by decomposing the function $Y_j(t)$ using cubic B-spline basis functions with $K$ interior knots. If $Y_j(t)$ belongs to the $h$th component, i.e., $z_{jh}=1$, $l_{\mathcal{X}}$ is the complete-data log-likelihood, and $D_j=l_{\mathcal{X} \setminus \mathbf{b}_j}-l_{\mathcal{X}}$, then $D_j \mid (z_{jh}=1) $ has an approximate shifted and scaled beta density, i.e.,
	\begin{equation}
		\label{eq:fundist}
		D_j \mid (z_{jh}=1) \sim f_{\text{beta}}\left(\frac{2n_h}{(n_h-1)^2} (d_j-c)~\bigg|~\frac{K+4}{2},\frac{n_h-K-5}{2}\right)
	\end{equation}
	for $c<d_j<\frac{(n_h-1)^2}{2n_h}+c, n_h>K+5$, where $c=-\log\hat{\pi}_h+\frac{K+4}{2}\log(2\pi)+\frac{1}{2}\log\abss{\sampcov_h}$, {$n_h$ is the number of points in component $h$}, $\hat{\pi}_h=n_h/n$, $$\sampcov_h=\frac{1}{n_h-1}\sum_{i=1}^n z_{ih}(\mathbf{b}_i-\bar{\mathbf{b}}_h)(\mathbf{b}_i-\bar{\mathbf{b}}_h)^\top$$ is the sample covariance matrix of component $h$,  $\bar{\mathbf{b}}_h=\frac{1}{n_h}\sum_{i=1}^n z_{ih}\mathbf{b}_i,$ and $K$ is the number of interior knots in the cubic spline basis used to generate $\{\mathbf{b}_1,\ldots,\mathbf{b}_n\}$.
\end{lemma}
\begin{proof}
	The proof requires the same assumptions as in \cite{clark24}, notably that the clusters are well-separated and that the number of observations in each component is large. As in \cite{clark24}, these assumptions can be relaxed in practice. For a further discussion on the effect of sample size, see Section~\ref{sec:senssim}.
	
	The fitted $(K+4)$-dimensional coefficients $\{\mathbf{b}_1,\ldots,\mathbf{b}_n\}$ are i.i.d.\ and assumed to arise from a finite Gaussian mixture model according to \eqref{eq:fitbmix2}. Using $\{\mathbf{b}_1,\ldots,\mathbf{b}_n\}$ in place of $\{\vecx_1,\ldots,\vecx_n\}$, $\vecSigma_g+\sigma^2(\mathbf{B}^\top\mathbf{B})^{-1}$ in place of $\vecSigma_g$, and $p=K+4$ in Theorem~1 of \cite{clark24} yields the desired result. 
\end{proof}

Using the density from \eqref{eq:fundist} we can generate the unconditional density of $D$ using a mixture model, i.e.,
\begin{equation}
	f(d\mid\vecvartheta)=\sum_{g=1}^{G}{\pi}_g  f_g(d \mid \vectheta_g),
	\label{eq:mixdensfun}
\end{equation}
where $f_g(d \mid \vectheta_g)$ is the shifted and scaled beta density described in \eqref{eq:fundist}, and $\vectheta_g=\{n_g,K,\hat{\pi}_g, \bar{\mathbf{b}}_g, \sampcov_g\}$.
\begin{remark}\label{rem:fun}
	The density in \eqref{eq:fundist} describes $D$ under the assumption that each $\mathbf{b}_i$ arises from a finite Gaussian mixture model and that there are no outlying coefficients.  If  \eqref{eq:fundist} fails to describe $D$, then we can conclude that a model assumption has been violated. In this case, we assume that the model is correctly specified apart from the presence of outliers. \end{remark}

\subsection{Multivariate and Functional Outliers}

While the distribution in Section~\ref{sec:fundist} used the complete-data log-likelihood, we now change to using the traditional log-likelihood, given by 
\begin{equation}
	\ell_{\mathcal{X}}(\vecvartheta)=\sum_{i=1}^n \log \left[\sum_{g=1}^G \pi_g \phi(\mathbf{b}\mid \vecmu_g,\vecSigma_g+\sigma^2(\mathbf{B}^\top\mathbf{B})^{-1})\right].
	\label{eq:loglikb}
\end{equation}
This choice is made because it is outputted by most clustering algorithms. In addition, \cite{clark24} show that $\ell_{\mathcal{X}}\rightarrow l_{\mathcal{X}} $ as the clusters separate. The funOCLUST algorithm removes candidate outliers iteratively until the subset log-likelihoods of the B-spline coefficients conform to the derived distribution. It is thus necessary to define a candidate outlier. Formally defined in Definition~\ref{def:funout}, the most outlying coefficient is the one whose removal yields the greatest increase in log-likelihood, i.e., the model improves the most in its absence.
\begin{definition}[Candidate Outlier] We define our candidate outlier as $\mathbf{b}_{k}$, where
	\begin{equation*}
		k=\argmax_{j \in [1,n]} \ell_{\mathcal{X} \setminus \mathbf{b}_j}, 
		\label{eq:argmax}
	\end{equation*}
	and $\ell_{\mathcal{X} \setminus \mathbf{b}_j}$ is the log-likelihood of the subset with the point $\mathbf{b}_j$ removed.
	\label{def:funout}
\end{definition}
Outlyingness in funOCLUST is defined through the finite-dimensional coefficient representation. Because the mapping
\begin{equation}
	Y_i(t)= \mathbf{B}(t)\mathbf{b}_i+e_i
\end{equation}
is determined by the coefficient vector $\mathbf b_i$, functions corresponding to coefficient vectors identified as outlying are classified as functional outliers. 

Relating this back to the functional space and the taxonomy of \cite{hubert15}, if a coefficient $\mathbf{b}_i$ is unusual in a single entry, this would affect only a narrow region, creating an isolated outlier. Alternatively, if each entry in $\mathbf{b}_i$ is increased or decreased by a similar amount, this would create shift outliers by moving the entire curve up or down. Amplitude outliers occur if a coefficient is multiplied by a constant $a\neq 1, a>0$, with $0\leq a<1$ having a dampening effect and $a>1$ having an amplifying effect. Finally, combinations of the above transformations would create shape outliers, including where $a<0$. 

\subsection{funOCLUST Algorithm}

The funOCLUST algorithm proceeds in two steps: filtering and then clustering. First, the functions $Y_1(t), \ldots, Y_n(t)$ are filtered by decomposing them into linear combinations of cubic B-splines. In the second step, the funOCLUST algorithm clusters the fitted coefficients $\mathbf{b}_1,\ldots,\mathbf{b}_n$ and detects outliers. In each iteration of the second step, subsets of the coefficients are clustered and their corresponding log-likelihood recorded. The distribution of these subset log-likelihoods is measured against the theoretical distribution with the KL divergence. The candidate outlier is then removed before the next iteration begins. The model is free of outliers when KL divergence is minimized. Formally, the funOCLUST algorithm is given in Algorithm~\ref{alg:funoclust}. 
\begin{algorithm}[!htb]
	\caption{funOCLUST algorithm}\label{alg:funoclust}
	\begin{algorithmic}[1]
		\Procedure{funOCLUST}{$\mathbf{t},\{Y_i(\mathbf{t})\}_{i \in [1,n]}, \{\xi_k\}_{k \in [0,K+1]},G,F$}
		\State \begin{varwidth}[t]{\linewidth} Generate a cubic B-spline basis, $\mathbf{B}$, defined by knots $\{\xi_k\}_{k \in [0,K+1]}$ and \par 
			evaluated at $\mathbf{t}$. 
		\end{varwidth}\newline
		\State\begin{varwidth}[t]{\linewidth} Use OLS regression to fit the coefficients for each function, \par 
			i.e., $\mathbf{b}_i=(\mathbf{B}^\top\mathbf{B})^{-1}\mathbf{B}^\top\mathbf{Y}_i$. 
		\end{varwidth}\newline
		\State \begin{varwidth}[t]{\linewidth} Run the OCLUST algorithm, available in the \texttt{oclust} package \par 
			\citep{clark25} for \textsf{R} \citep{R25}, with \par 
			$\mathcal{X}=\{\mathbf{b}_1, \ldots, \mathbf{b}_n\}$, $G$ clusters, and $F$ maximum outliers. 
		\end{varwidth}\newline
		\EndProcedure
	\end{algorithmic}
\end{algorithm}

\begin{remark}
	As with OCLUST \citep{clark24}, the number of clusters $G$ is assumed to be specified in advance. In practice, domain knowledge can often be used to determine an appropriate value of $G$. When such information is unavailable, selecting the number of clusters remains an open problem for funOCLUST.
\end{remark}

\begin{remark}
	The choice of $F$ is user-specified and denotes the maximum number of points removed during the outlier search. In the absence of domain knowledge, $F=n/2$ is a conservative choice and represents the scenario where up to 50\% of functions could be candidate outliers.
\end{remark}

\begin{remark}
	The basis is assumed to be specified in advance. In practice, it may be selected using model selection strategies such as cross-validation. In Section~\ref{sec:funsim}, we employ leave-one-out cross-validation (LOOCV), fitting the basis using $t_j-1$ time points and predicting the held-out observation. This procedure is repeated over all time points and all functions, and the number of knots is chosen to minimize the resulting prediction error.
\end{remark}

\section{Simulation Study} \label{sec:funsim}
In this simulation study, we  compare the funOCLUST algorithm to five competitors:
funHDDC;
T-funHDDC;
a robust curve clustering approach called tkmeans \citep{garcia05};
the approach of \cite{james03}, which we call mocca due to its {\sf R} implementation; and functional outlyingness \citep[fOutl;][]{hubert15,mrfDepth}.

The mocca, funHDDC and {T-funHDDC} methods are functional clustering algorithms, with funHDDC and T-funHDDC specializing in high-dimensional cases. Both tkmeans and funOCLUST cluster and detect outliers simultaneously, while fOutl detects outliers only and does not cluster the functional data. 
\subsection{Setup}
The five methods described above and funOCLUST are applied to 100 generated datasets under eight scenarios as a fractional factorial design testing the number of clusters (2 vs.\ 5), functional complexity (moderate vs.\ high), sampling sparsity (dense vs.\ sparse), and outlier generation mechanism (shift-scale vs.\ multivariate-t). The contamination level is fixed at 20\% for the shift-scale outliers, while the multivariate-t contamination inherently induces heavy-tailed deviations. Each scenario is described in Table~\ref{simscen}.

\begin{table}[ht]
	\centering
	\caption{Descriptions of each simulation scenario with respect to the number of clusters, functional complexity, observation sparsity, and outlier type.}\label{simscen}
	\begin{tabular}{cccccc}
		\hline
		scenario & clusters & complexity & sparsity & outlier type & contamination \\ 
		\hline
		1 & 2 & moderate & dense & heavy-tail & \\ 
		2 & 5 & moderate & dense & shift-scale & 0.20 \\ 
		3 & 2 & high & dense & shift-scale & 0.20 \\ 
		4 & 5 & high & dense & heavy-tail  &  \\ 
		5 & 2 & moderate & sparse & shift-scale & 0.20 \\ 
		6 & 5 & moderate & sparse & heavy-tail  & \\ 
		7 & 2 & high & sparse & heavy-tail  & \\ 
		8 & 5 & high & sparse & shift-scale & 0.20 \\ 
		\hline
	\end{tabular}
	
\end{table}

In each replicate, 300 curves are simulated according to the equation
\begin{equation}
	x_{i,g}(t) = \mu_g(t) + \varepsilon(t)
\end{equation}
for $t$ on an equally-spaced grid of size 80 between 0 and 1. In the moderate complexity scenario, the mean functions are:
\begin{gather*}
	\mu_1(t) = 2t + 0.5, \quad	\mu_2(t) =-1.5(t-0.5)^2 + 1, \quad \mu_3(t) =\exp\{-3t\}, \\ \mu_4(t) =\frac{1}{1 + \exp\{-10(t-0.5)\}}, \quad
	\mu_5(t) = \max(0, 2t - 1),
\end{gather*}
representing linear, quadratic, exponential, sigmoid, and piecewise curves, respectively. In the high-complexity scenario, the mean functions are:
\begin{gather*}
	\mu_1(t) = 0.5t + \exp\{-(t-0.3)^2/0.01\}, \quad	\mu_2(t) =-t^2+1+ \exp\{-(t-0.7)^2/0.02\},\\
	\mu_3(t) =\exp\{-2t\} + 0.3\sin(4\pi t), \quad \mu_4(t) =\frac{1}{1 + \exp\{-10(t-0.5)\}} + \exp\{-(t-0.4)^2/0.01\} ,\\
	\mu_5(t) = \max(0, 1.5t - 0.5)+\exp\{-(t-0.6)^2/0.015\},
\end{gather*}
which are also linear, quadratic, exponential, sigmoid, and piecewise, respectively, but with added bumps and periodicity. Figure~\ref{fig:meanfuncs} plots these mean curves. 

\begin{figure}[!ht]
	\includegraphics[width=\textwidth]{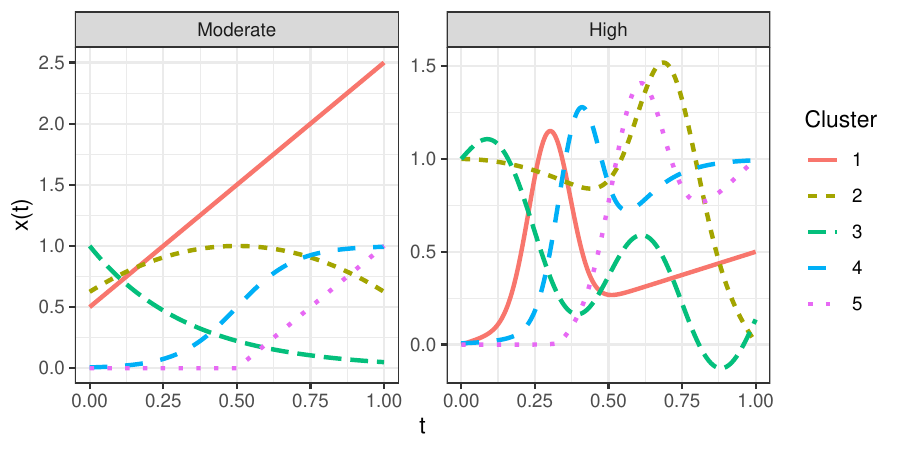}
	\caption{Mean functions for the simulated data.}\label{fig:meanfuncs}
\end{figure}

To induce sparsity in Scenarios 5--8, the same time points are removed from all curves in a dataset. For each dataset, the number of missing time points is randomly drawn from a discrete uniform distribution between 30 and 50, and the corresponding time points are deleted for every curve. This results in approximately 50\% missingness on average. Such a mechanism may arise in longitudinal studies where measurements are scheduled at irregular, but common times for all subjects.

Finally, outliers are introduced in two ways. The shift-scale type is created by randomly selecting 20\% of the curves and scaling them by a random factor generated uniformly in $1.5<a<2.5$ and shifting them by a similar factor $-2 < b < 2$, i.e., \newline $x^*_i  = ax_i(t) + b$. For the shift-scale type, $\varepsilon_i(t) \sim N(\vecmu = \mathbf{0}, \vecSigma = \mathbf{I}_{t_j})$. The second type introduces outliers by generating the functions with a heavy-tailed error such that $\varepsilon_i(t) \sim T(\vecmu = \mathbf{0}, \vecSigma = \mathbf{I}_{t_j}, \nu = 10)$.  A two-cluster example of both types of generation scenario is shown in Figure~\ref{fig:datex}.
\begin{figure}[!ht]
	\includegraphics[width=\textwidth]{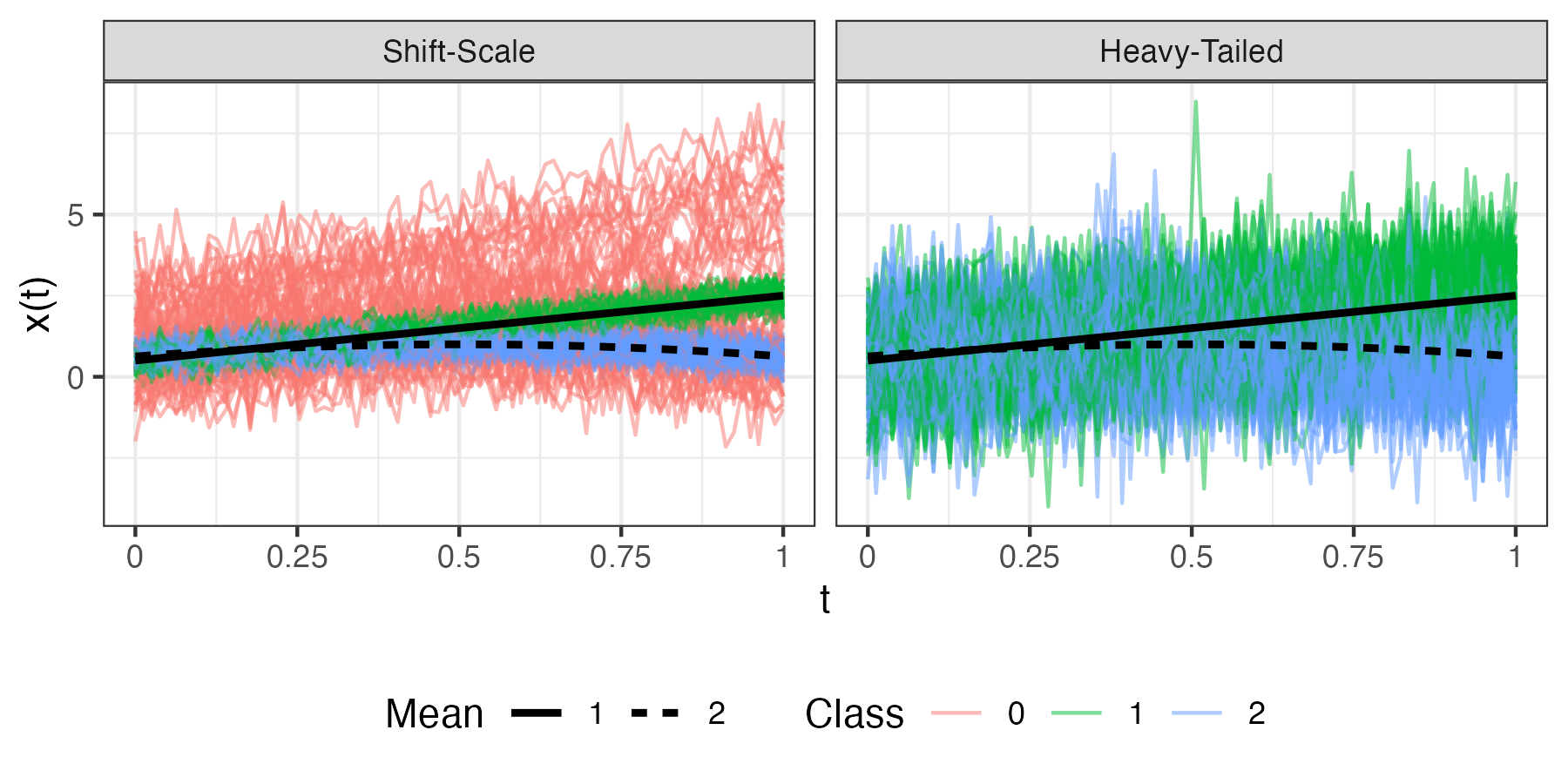}
	\caption{Examples of both types of outlier scenario on a dataset with two clusters, moderate difficulty, and dense observations.}\label{fig:datex}
\end{figure}

In each run, a cubic B-spline basis is chosen with the number of knots selected by leave-one-out cross-validation (LOOCV) using \texttt{lm} and the \texttt{loocv} function from the \texttt{switchSelection} \textsf{R} package \citep{potanin26} to minimize the root mean squared error. This basis is then used for all methods apart from fOutl, which does not require it. The funOCLUST approach is run with chosen upper bound $F$ of 150, representing 50\% of the curves. For tkmeans, the basis coefficients are clustered according to a trimmed $k$-means algorithm, using the \texttt{tclust} package \citep{fritz12} for {\sf R}. The \texttt{ctlcurves} function is used to choose the trimming percentage and $\alpha$ is chosen where the second derivative of the objective function is minimized. The funHDDC and {T-funHDDC} algorithms are run with the {\sf R} packages \texttt{funHDDC} \citep{schmutz21} and \texttt{TFunHDDC} \citep{anton23}, respectively, each with 20 random $k$-means starts, all models, and with the Catell-scree threshold  in \{0.05, 0.1, 0.2, 0.6\}. For each package, the final model is chosen automatically using the Bayesian information criterion \citep[BIC;][]{schwarz78}. The mocca algorithm is run with the \texttt{fdaMocca} package, setting the number of knots consistent with the previous algorithms. The fOutl method is implemented with the \texttt{fOutl} function in the \texttt{mrfDepth} package \citep{mrfDepth}, using adjusted outlyingness with Stahel-Donoho outlyingness \citep{stahel81, donoho82} as the distance option. 

\subsection{Clustering Results}
Each clustering algorithm is evaluated using the adjusted Rand index \citep[ARI;][]{hubert85}. The ARI compares two partitions --- in this case, real and predicted classes. The ARI equals one under perfect classification and has expected value zero under random class assignment. For funOCLUST and tkmeans, points identified as ``outliers" do not have class labels.  To compare these methods to their competitors, we reclassify the points according to their final models, with tkmeans assigning outliers to the nearest cluster centre and funOCLUST assigning outliers to the one which produces the largest posterior probability. The ARIs for each scenario are shown in Figure~\ref{fig:ARIs}. 
\begin{figure}[!ht]
	\includegraphics[width=\textwidth]{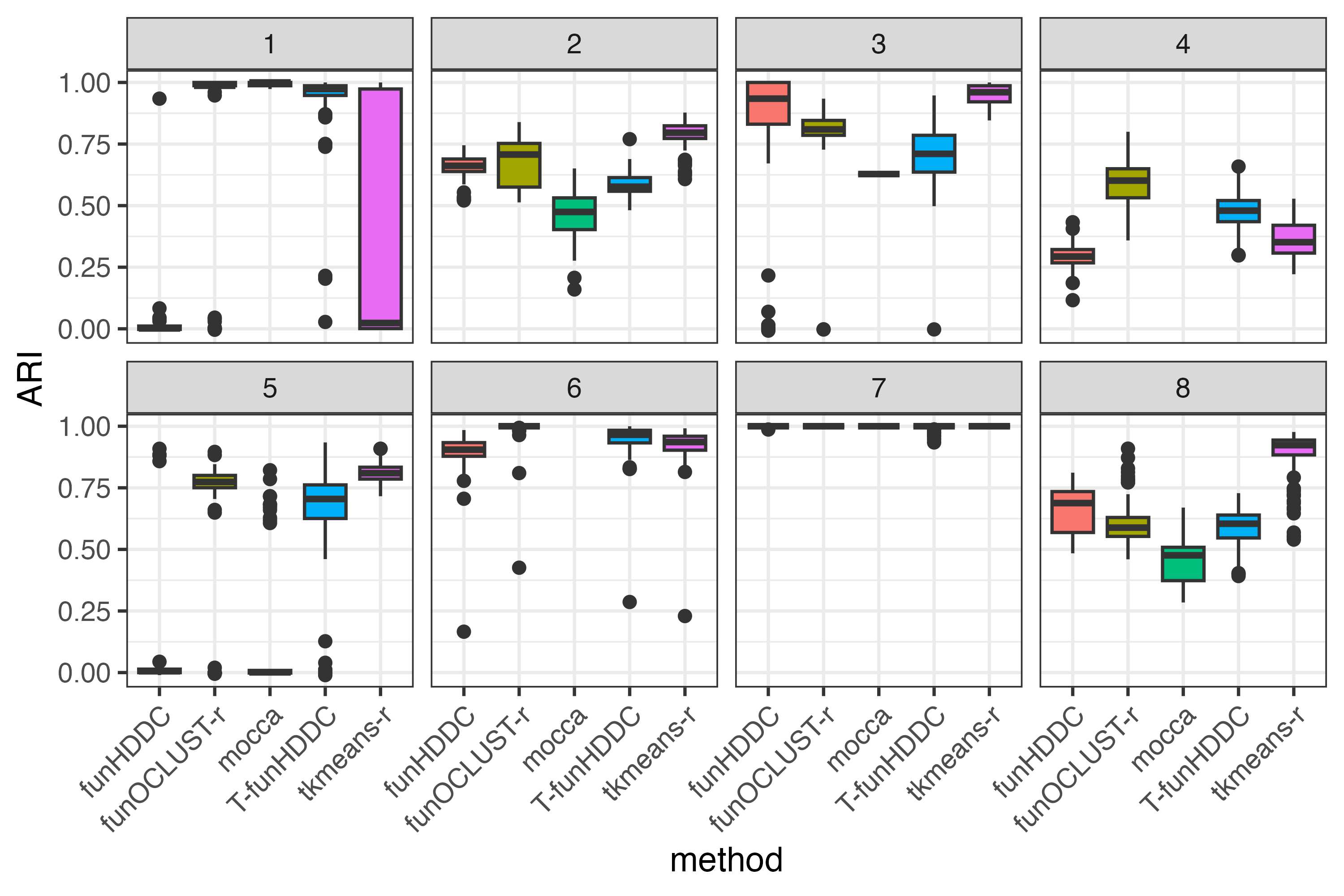}
	\caption{Boxplots of ARI for each simulation scenario. 	The indicator `-r' indicates that ARI is calculated after outliers are reclassified. }\label{fig:ARIs}
\end{figure}

The funOCLUST algorithm performs very well in scenarios 1, 4, 6, and 7, corresponding to when the errors have heavy tails. In the shift-scale scenarios tkmeans dominates, while funOCLUST produces similar results to its other competitors. This suggests that funOCLUST is better suited for heavy-tailed errors but can return good, consistent results in a variety of situations. 

The mocca algorithm performs among the best in scenarios 1 and 7, but underperforms everywhere else, including in scenarios 4 and 6 where it is unable to fit a solution. It seems to be suited to the two-cluster case with heavy-tail errors. Unsurprisingly, T-funHDDC is more robust in the heavy-tailed scenarios because it clusters with the multivariate-t distribution. 

Finally, funOCLUST, tkmeans, and fOutl are evaluated on their outlier identification accuracy. We evaluate only scenarios with shift-scale outliers because the identities of the outlying functions are known. Functions are considered outliers by fOutl if their outlyingness exceeds the cutoff from the \texttt{fom} function in the \texttt{mrfDepth} package. 

Figure~\ref{fig:funerr} shows false positive rates (dark pink) and false negative rates (dark blue). All methods have very low false positive rates, indicating that they are conservative and rarely classify `good' functions as outliers. The funOCLUST algorithm has the lowest false negative rate in scenarios 2 and 5, where the functions are of moderate complexity. Notably, funOCLUST has a 51\% false negative rate in scenario 8, where there are five clusters and 20\% contamination, leading to each cluster only having 48 non-outlying curves. While the competing methods are non-parametric and therefore less sensitive to distributional assumptions, the combination of high complexity, sparse sampling, and limited per-cluster information makes outlier detection inherently difficult. Despite this, funOCLUST still delivers solid results with respect to ARI.

\begin{figure}[!ht]
	\centering
	\includegraphics[width=\textwidth]{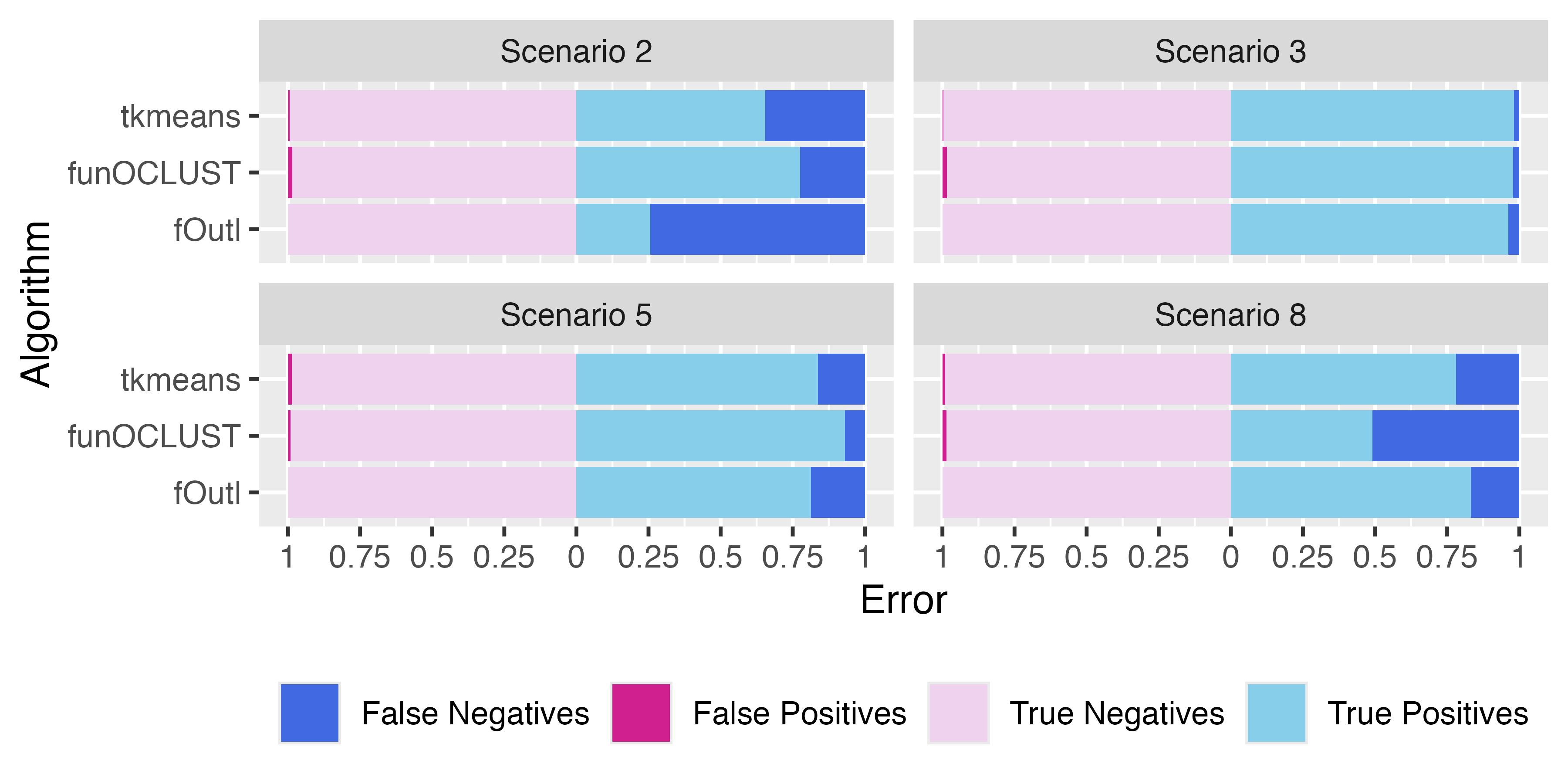}
	\caption{Outlier identification rates for funOCLUST, tkmeans, and fOutl on the simulated datasets with shift-scale outliers.}\label{fig:funerr}
\end{figure}

\subsection{Main Effects}

Due to the balanced nature of the design, we are able to isolate the effects of the number of clusters, functional complexity, sparsity, and outlier type on the clustering performance of funOCLUST and its competitors. A plot of mean ARI is presented in Figure~\ref{fig:maineffects}. As expected, funOCLUST performs best in settings with fewer clusters and moderate functional complexity. However, when the algorithms are considered together, they achieve better performance on more complex datasets. Similar to what is seen in Figure~\ref{fig:ARIs}, funOCLUST attains the best results when outliers are generated with a heavy-tailed distribution. Interestingly, the algorithms as a whole perform better when time points are sparsely sampled. One possible explanation is that the signal may be less obscured by noise. Importantly, funOCLUST achieves the largest mean ARI across scenarios, showing its robustness in a variety of situations.

\begin{figure}[!ht]
	\centering
	\includegraphics[width=4in]{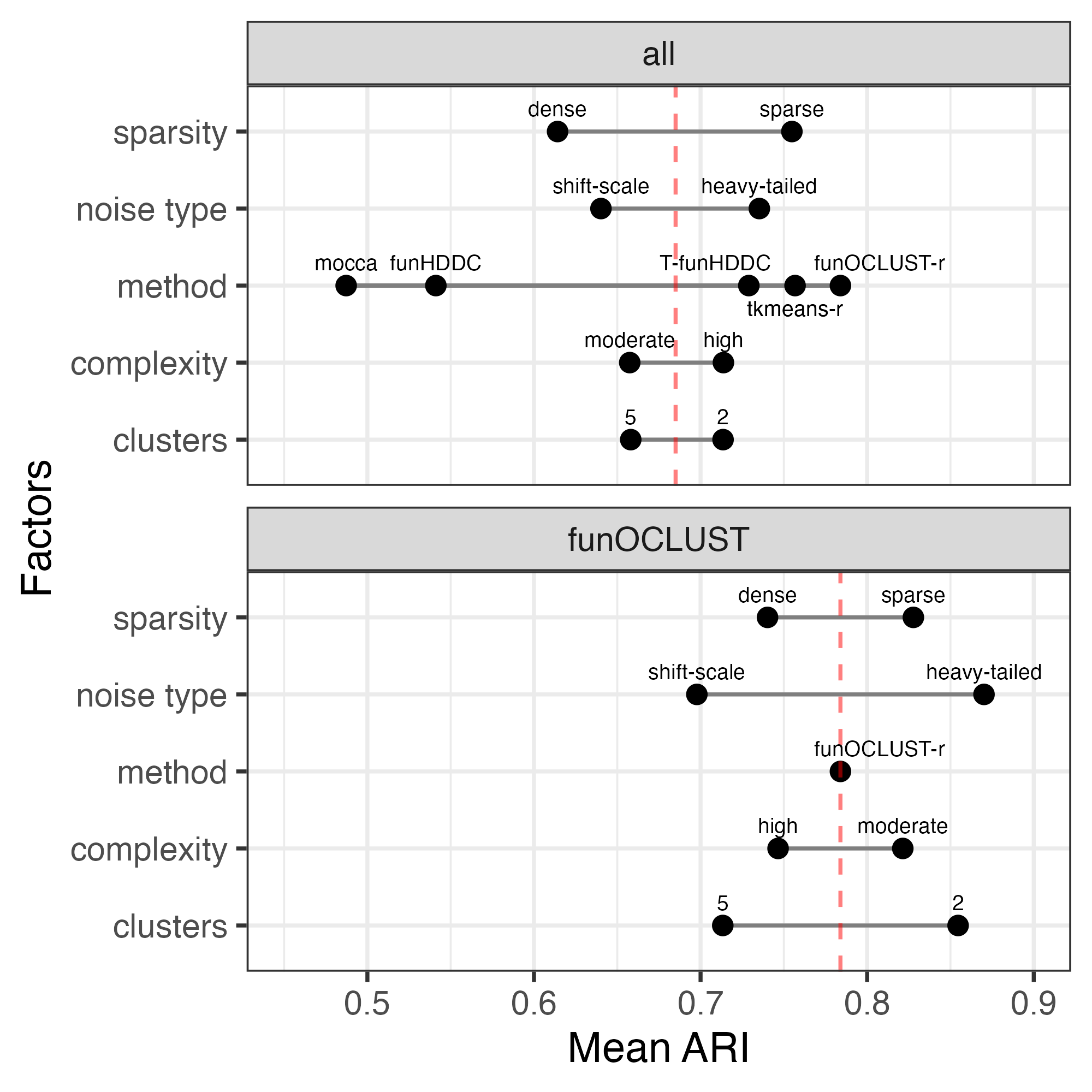}
	\caption{Main effect plot showing the mean ARI for each level of factor. The indicator `-r' indicates that ARI is calculated after outliers are reclassified. }\label{fig:maineffects}
\end{figure}

\section{Sensitivity Study} \label{sec:senssim}

We further test funOCLUST's sensitivity to sample size and contamination level by generating 50 data sets with dense sparsity, shift-scale noise, moderate complexity, and two clusters for each combination of $n\in (100,300,500,1000)$ and contamination level $\in (0,0.05,0.1,0.2,0.3,0.4)$.  The funOCLUST algorithm is run with $G=2$, $F=n/2$, and the bases are determined by LOOCV. Means for false positive, false negative, and misclassification rates are plotted in Figure~\ref{fig:heatmaps}. Misclassification rate considers the three-class solution, where ``outlier" is its own class.

\begin{figure}[!ht]
	\centering
	\includegraphics[width=6in]{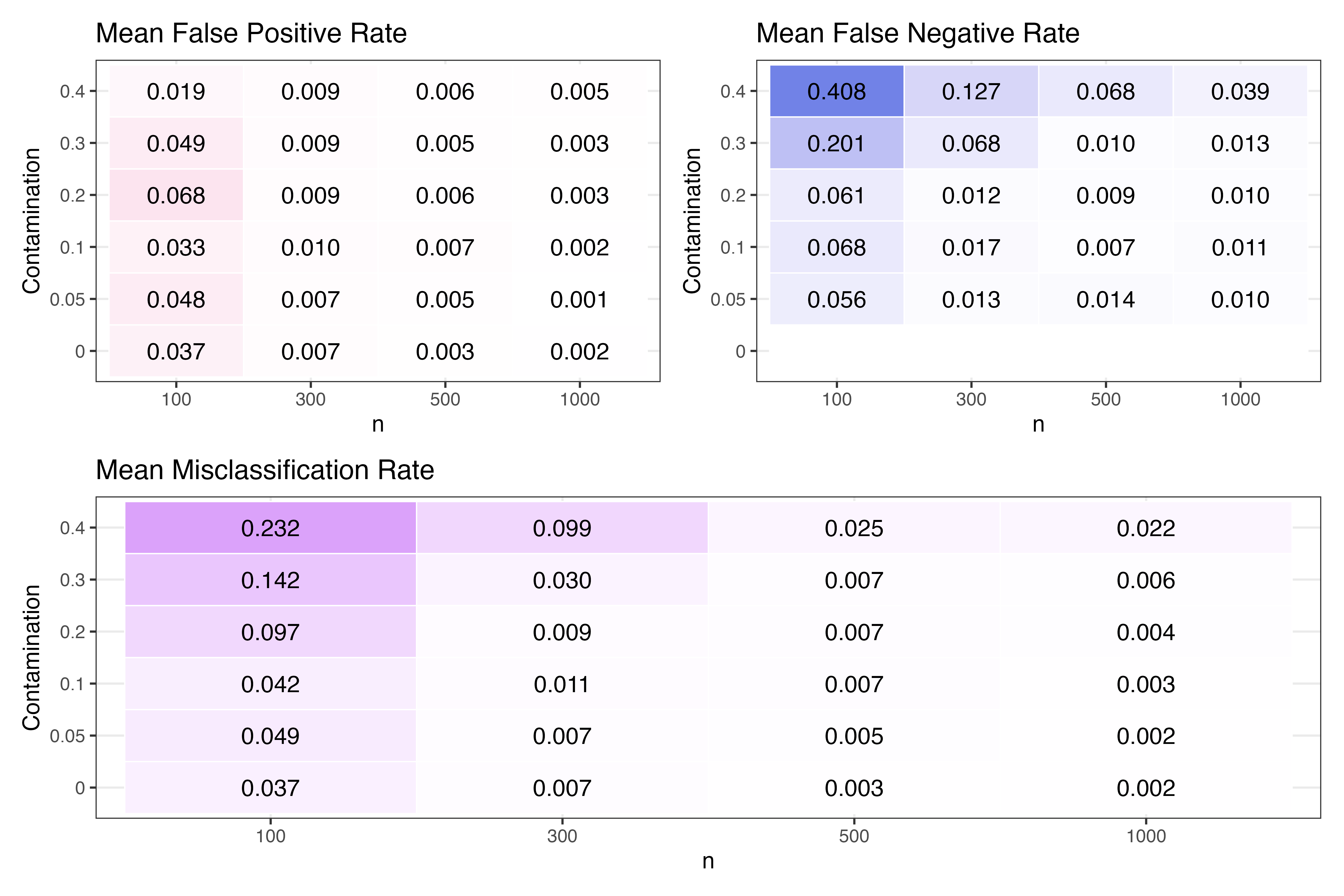}
	\caption{Error rates for funOCLUST in the sensitivity analysis.}\label{fig:heatmaps}
\end{figure}

The funOCLUST algorithm achieves a negligible misclassification rate when $n \geq 300$ and the contamination level is at most 0.20. Outside this region, the misclassification rate ranges from 2.2\% to 23.2\%, with the highest rate observed when contamination is greatest and $n$ is smallest. This corresponds to a particularly challenging setting in which, on average, each dataset contains 40 outlying curves and only 30 non-outlying curves per cluster. This is supported by the fact that in the majority of cases when $n=100$ and contamination is 40\%, funOCLUST merges the two clusters into one and identifies a subset of outliers as the second cluster. 

The false positive rate is negligible for $n \geq 300$ and remains low even when $n=100$. The false negative rate exhibits trends similar to those of the overall misclassification rate, increasing as contamination rises and sample size decreases. These results suggest that funOCLUST rarely identifies non-outlying observations as outliers and is more conservative in settings with smaller sample sizes or higher levels of contamination.

The strongest performance is observed for moderate to large sample sizes, which is consistent with the large-sample assumptions underlying OCLUST-based methods. Consequently, the reduced performance observed at $n=100$ is not unexpected and remains reasonable given the some of the substantial levels of contamination considered.

\section{Real Data Examples}\label{sec:funreal}

\subsection{Pedestrian Traffic}
Consider real data on pedestrian traffic in Melbourne, Australia, where sensors installed around the city capture the number of pedestrians passing by each hour. The data are available on \texttt{kaggle.com} \citep{chinatown} for all sensors between the years 2009 and 2022.  Data captured for the year 2017 from the Chinatown--Swanston St.\ (North) sensor are plotted in Figure~\ref{fig:chinatowntrue}, with observations coloured according to whether the data correspond to a workday or non-workday (i.e., a weekend or public holiday). There are 365 observations, each with values collected at 24 time points. Of the 365 days captured, 249 are workdays and 116 are not. There is one missing value on October 1st between the hours of 2am and 3am. 
\begin{figure}[!htb]
	\includegraphics[width=\textwidth]{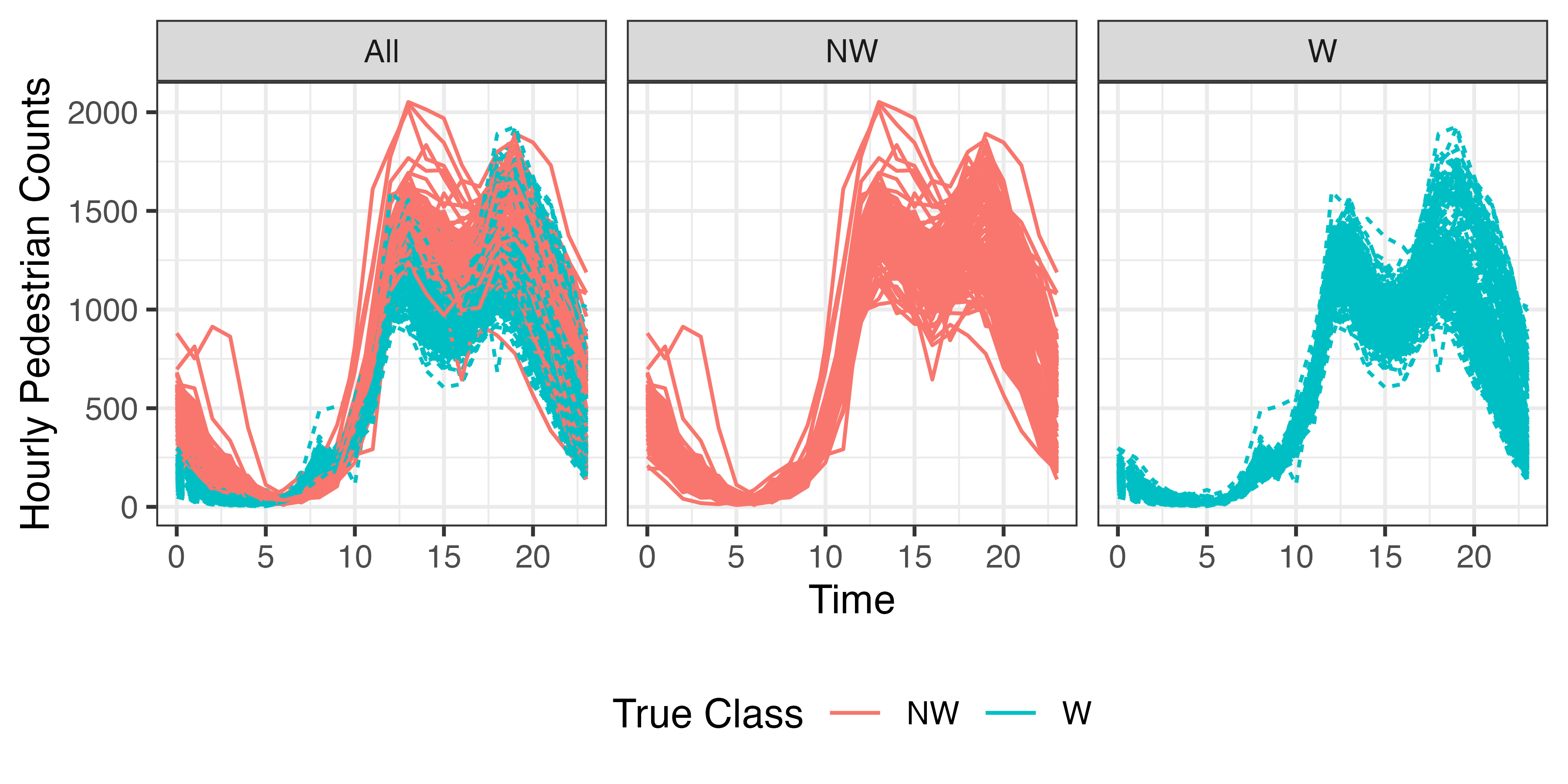}
	\caption{Pedestrian traffic recorded hourly for one year in Chinatown, Melbourne. Curves are coloured according to whether each day is a workday (W) or non-workday (NW). }\label{fig:chinatowntrue}
\end{figure}

Each algorithm from Section~\ref{sec:funsim} is run on these data. The optimal number of cubic B-spline bases is determined to be 12 by LOOCV.  For funOCLUST, the upper bound $F$ is set to 50. The missing value is imputed as the mean of all functions at that same time of day. The clustering results are shown in Table~\ref{tab:chinatowncomp} as confusion matrices.   
\begin{table}[!htb]
	\caption{Classification results for each algorithm on the Chinatown dataset. The indicator `--r' indicates the clustering solution after outliers are reclassified.}\label{tab:chinatowncomp}
	\centering
	\begin{tabular}{@{\extracolsep{\fill}}llccccccccccc}
		
		&   \multicolumn{3}{c}{funOCLUST}&& \multicolumn{2}{c}{funOCLUST--r}&&\multicolumn{2}{c}{{funHDDC}}&&\multicolumn{2}{c}{T-funHDDC}\\ 
		\cmidrule{2-4}\cmidrule{6-7}\cmidrule{9-10}\cmidrule{12-13}
		{Channel}&1&2&bad&&1&2&&1&2&&1&2\\
		\hline 
		Workday &244&0&5&&249&0&&175&74&&202&47\\ 
		Non-Workday &0&99&17&&1&115&&13&103&&114&2\\ 
		\hline
		&&&&&&&&&&&&\\
		&   \multicolumn{3}{c}{tkmeans}&& \multicolumn{2}{c}{tkmeans--r}&&\multicolumn{2}{c}{{mocca}}&&&\\ 
		\cmidrule{2-4}\cmidrule{6-7}\cmidrule{9-10}
		{Channel}&1&2&bad&&1&2&&1&2&&&\\
		\cmidrule{1-10}
		Workday &189&60&0&&189&60&&184&65&&&\\ 
		Non-Workday &48&58&10&&52&64&&49&67&&&\\ 
		\hline
	\end{tabular}	
	
\end{table} 

After removing 22 outliers, funOCLUST is able to recover day type for the remaining curves without error. Included in the list of outliers are January~1 (New Year's Day), December~25 (Christmas Day),  December~26 (Boxing Day), and January~28--29 (Chinese New Year). After reassigning the outliers according to the maximum posterior class, funOCLUST misclassifies only one non-workday as a workday. Interestingly, this misclassified day is December 26th, Boxing Day. Although officially a non-workday, Boxing Day in Melbourne is marked by major sporting events and high shopping activity, which may cause its pedestrian activity to resemble that of a workday. The funOCLUST solution is plotted in Figure~\ref{fig:chinatownpred}. Many of the outlying observations have higher-than-average traffic in the afternoon (noon to 5pm) or early morning (midnight to 4am). Additionally, these outliers often exhibit unusual combinations of traffic patterns. For example, traffic volume around noon and 7pm may resemble a typical workday, while the volume around 4pm is more similar to what is observed on a weekend or holiday. Because their unusual shape occurs over broad portions of the curve rather than within a narrow domain, they are classified as persistent outliers in the functional outlier taxonomy.

\begin{figure}[!htb]
	\includegraphics[width=\textwidth]{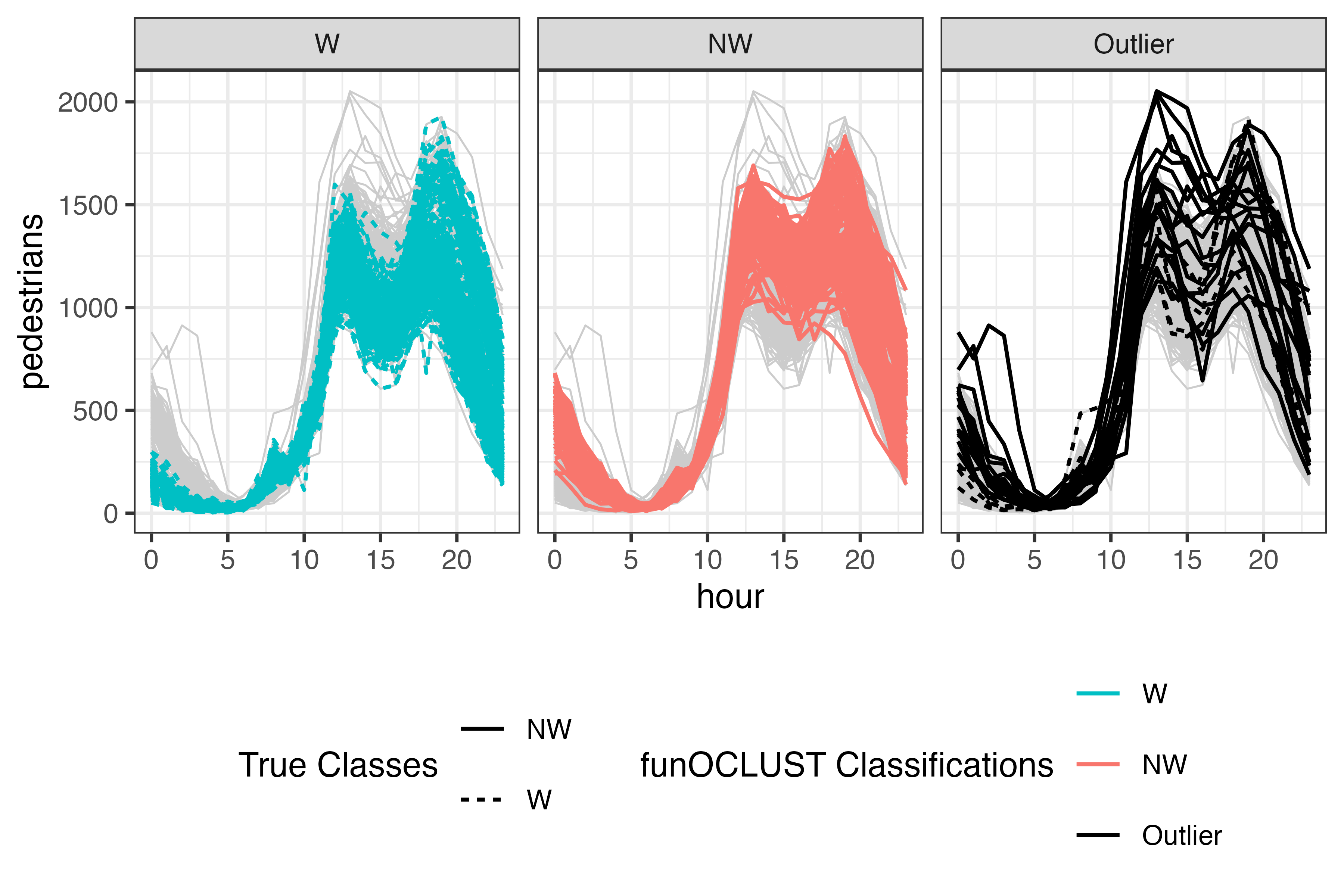}
	\caption{Pedestrian traffic recorded hourly for one year in Chinatown, Melbourne, where curves are coloured according to the funOCLUST solution.}\label{fig:chinatownpred}
\end{figure}

\subsection{NOx Benchmark}

Finally, we apply funOCLUST to the NO$_\text{x}$ data as a benchmark dataset. This dataset has been widely used in the functional model-based clustering literature, including by \cite{smith22, anton23} and \cite{rivera19}, providing a natural basis for comparison with existing methods. 

First used in \cite{febrero08}, the data reflect hourly measurements of nitric oxide (NO) and nitrogen dioxide (NO$_2$)  in the air over 115 days in Barcelona, Spain during 2005. Curves are separated by day with 76 workdays and 39 non-workdays. The raw data are plotted in Figure~\ref{fig:poblenou}, with the group-wise hourly mean plotted on top.

\begin{figure}[!htb]
	\includegraphics[width=\textwidth]{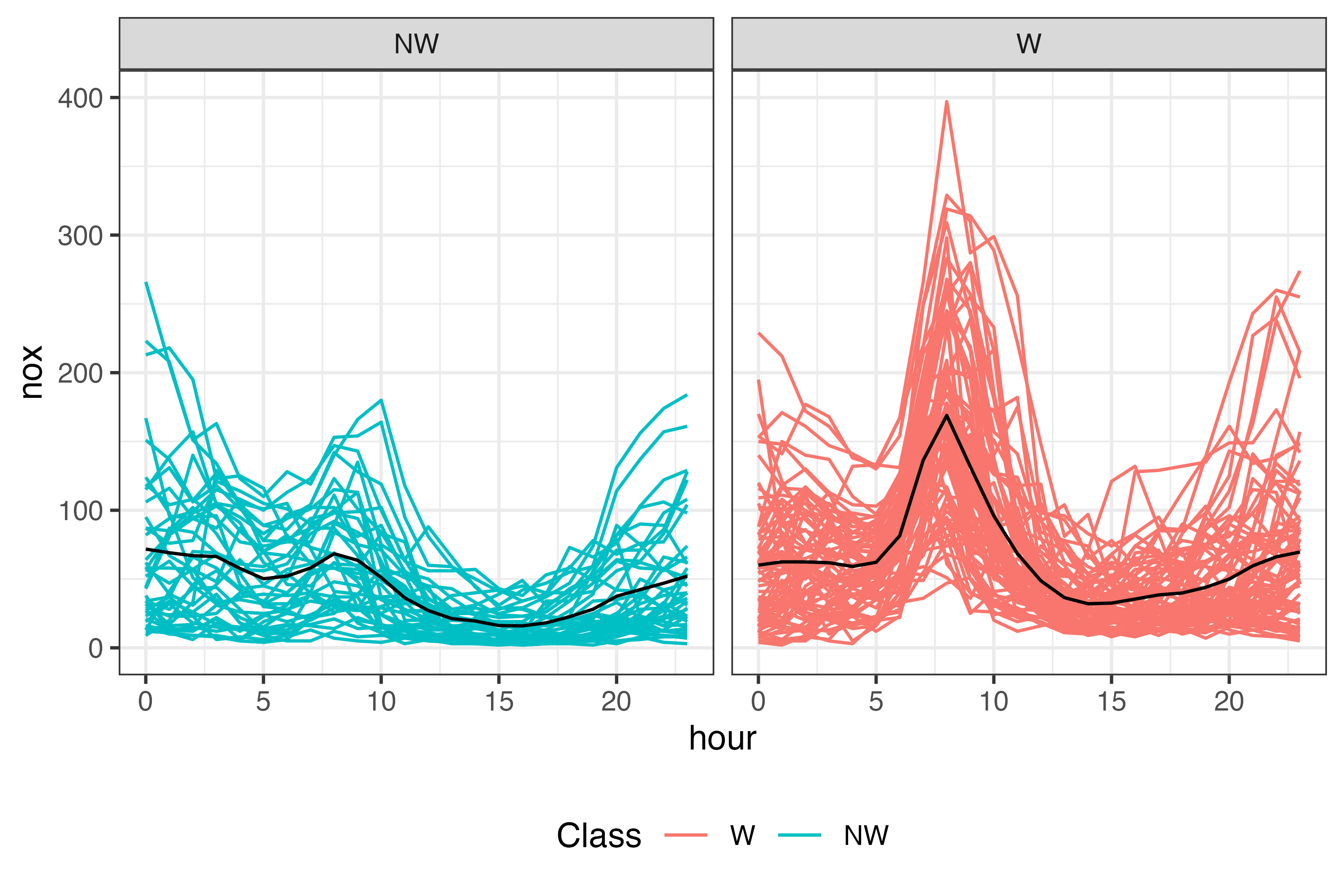}
	\caption{NO$_\text{x}$ concentration in Barcelona, Spain in 2005. Curves are coloured according to whether each day is a workday (W) or non-workday (NW). Solid black curves indicate the group-wise hourly mean.}\label{fig:poblenou}
\end{figure}

The funOCLUST method is run with all Gaussian parsimonious clustering models, two clusters and $F=57$ maximum outliers. Eight basis functions are used, as selected by LOOCV. Table~\ref{tab:benchmark} presents the CCR obtained by funOCLUST, alongside CCR values for competing methods reported in their corresponding methodological papers.  The ``EEE" covariance structure produces the best CCR (0.86), corresponding to clusters with equal volumes, shapes, and orientations. Furthermore, funOCLUST's range of CCR from 0.51 to 0.86 places its best-performing models on par with existing approaches while demonstrating the importance of the choice of covariance structure.

\begin{table}[!htb]
	\centering
	\caption{Correct classification rates (CCR). Values for funOCLUST are obtained in the present study, while values for competing methods are taken from their original methodological papers.}
	\begin{tabular}{lcc}
		\hline
		Method & Model & CCR \\
		\hline
		T-FunHDDC & &0.52 -- 0.91 \\
		C-FunHDDC    & &0.60 -- 0.70 \\
		RFC          & &0.66 -- 0.85 \\
		funOCLUST &  & 0.51 -- 0.86 \\
		& EII & 0.68 \\
		&VII & 0.76 \\
		&EEI & 0.65 \\
		&VEI & 0.78 \\
		&EVI & 0.69 \\
		&VVI & 0.72 \\
		&EEE & \textbf{0.86} \\
		&VEE & 0.83 \\
		&EVE & 0.54 \\
		&VVE & 0.78 \\
		&EEV & 0.51 \\
		&VEV & 0.55 \\
		&EVV & 0.57 \\
		&VVV & 0.65 \\
		\hline
	\end{tabular}
	\label{tab:benchmark}
\end{table}

\section{Discussion}
The OCLUST algorithm has been extended to functional data. Using an assumption that cubic B-spline coefficients are multivariate normally distributed, the distribution of subset log-likelihoods was derived. This is the basis for the funOCLUST algorithm, which proceeds in two stages. First, functions are transformed using a cubic spline basis, and then the OCLUST algorithm is run, removing candidate outliers one-by-one until the target distribution is reached. 

The funOCLUST algorithm showed good, consistent classification performance across a variety of simulation scenarios, achieving the largest mean ARI across scenarios. When applied to a real dataset on pedestrian traffic in Melbourne, funOCLUST identified certain significant dates as outliers and was able to recover whether the data came from a workday or weekend/holiday. On the NO$_\text{x}$ benchmark, the proposed method produced results on par with its comparitors. 

The funOCLUST algorithm marks the first extension to OCLUST to functional data. This paves the way for future two-stage extensions, including to multivariate and skewed functional data, and for an OCLUST in the model-based paradigm, where the functional decomposition is estimated within the expectation-maximization algorithm, allowing the model to adapt as outliers are removed. 

\bibliographystyle{chicago}

\end{document}